\documentclass[11pt]{article}

\usepackage{amssymb,amsmath}
\usepackage{color}
\usepackage{graphicx}
  \DeclareGraphicsExtensions{.jpeg,.png,.jpg,.eps}
\usepackage{tikz}
\usepackage{epstopdf}

\def\eb{{\bf e}}

\def\G{{\cal G}}

\def\N{{\cal N}}

\def\R{{\mathbb R}}

\def\al{\alpha}
\def\d{\delta}
\def\D{\Delta}
\def\e{\epsilon}
\def\g{\gamma}

\def\l{\lambda}
\def\L{\Lambda}

\def\r{\rho}
\def\s{\sigma}
\def\SI{\Sigma}
\def\t{\tau}
\def\th{\theta}

\def\bbeta{{\boldsymbol \beta}}

\def\ap{\rightarrow}

\def\seq{\subseteq}

\def\bz{{\bf 0}}

\def\fa{\; \forall}

\def\st{\mbox{ s.t. }}
\def\sg{\mbox{sign}}
\def\nm{\Vert}

\renewcommand{\iff}{\mbox{$\; \; \Longleftrightarrow \; \;$}}
\renewcommand{\and}{\mbox{$\wedge$}}

\newcommand{\bc}{\begin{center}}
\newcommand{\ec}{\end{center}}
\newcommand{\be}{\begin{equation}}
\newcommand{\ee}{\end{equation}}
\newcommand{\bd}{\begin{displaymath}}
\newcommand{\ed}{\end{displaymath}}
\newcommand{\ba}{\begin{array}}
\newcommand{\ea}{\end{array}}
\newcommand{\ben}{\begin{enumerate}}
\newcommand{\een}{\end{enumerate}}
\newcommand{\bit}{\begin{itemize}}
\newcommand{\eit}{\end{itemize}}
\newcommand{\beq}{\begin{eqnarray}}
\newcommand{\eeq}{\end{eqnarray}}
\newcommand{\btab}{\begin{tabular}}
\newcommand{\etab}{\end{tabular}}
\newcommand{\bfig}{\begin{figure}}
\newcommand{\efig}{\end{figure}}
\newcommand{\btp}{\begin{tikzpicture}}
\newcommand{\etp}{\end{tikzpicture}}

{\bf}{\it}
\newtheorem{definition}{Definition}{\bf}{\it}
{\bf}{\rm}
{\bf}{\it}
\newtheorem{theorem}{Theorem}{\bf}{\it}
{\bf}{\it}
{\bf}{\it}
{\bf}{\rm}
{\bf}{\rm}

\newcommand{\argmin}{\operatornamewithlimits{argmin}}

\newcommand{\nmm}[1]{ \nm #1 \nm }
\newcommand{\nmeu}[1]{ \nm #1 \nm_2 }
\newcommand{\nmeusq}[1]{ \nm #1 \nm_2^2 }

\newcommand{\nmp}[1]{ \nm #1 \nm_p }

\newcommand{\halmos}{\hfill $\Box$}

\newcommand{\supp}{\mbox{supp}}

\def\xh{\hat{x}}

\def\hl{h_{\L}}
\def\hlc{h_{\L^c}}
\def\hlo{h_{\L_0}}
\def\hloc{h_{\L_0^c}}

\def\nmsl1{\nm_{{\rm SL1}}}

\def\xh{\hat{x}}

\def\hl{h_{\L}}
\def\hlc{h_{\L^c}}
\def\hlo{h_{\L_0}}
\def\hloc{h_{\L_0^c}}

\def\Rcal{{\cal R}}

\newcommand{\nmcmu}[1]{ \nm #1 \nm_{C,\mu} }
\newcommand{\nmgmu}[1]{ \nm #1 \nm_{{\rm SGL},\mu} }

\def\rbar{\bar{\rho}}

\begin{document}



\title{
Two New Approaches to Compressed Sensing \\
Exhibiting Both Robust Sparse Recovery \\
and the Grouping Effect
}

\author{Mehmet Eren Ahsen, Niharika Challapalli and
Mathukumalli Vidyasagar
\thanks{This research was supported by the National Science Foundation under
Award \#ECCS-1306630 and the Cecil \& Ida Green Endowment at UT Dallas.
MEA is with the Mount Sinai School of Medicine, New York;
mehmeteren.ahsen@mssm.edu.
Niharika Challapalli is a Ph.D.\ student in the Electrical
Engineering Department at the University of Texas at Dallas;
niharikaiitbbs@gmail.com.
Mathukumalli Vidyasagar is with the Systems Engineering Department
at the University of Texas at Dallas, and the Indian Institute of 
Technology Hyderabad;
m.vidyasagar@utdallas.edu and m.vidyasagar@iith.ac.in.}
}



\maketitle

\begin{abstract}

In this paper we introduce a new optimization formulation for
sparse regression and compressed sensing, called
CLOT (Combined L-One and Two), wherein the regularizer is a convex
combination of the $\ell_1$- and $\ell_2$-norms.
This formulation differs from the Elastic Net (EN) formulation, in which the
regularizer is a convex combination of the $\ell_1$- and $\ell_2$-norm 
\textit{squared.}
It is shown that, in the context of compressed sensing, the EN formulation
\textit{does not achieve} robust recovery of sparse vectors,
whereas the new CLOT formulation achieves robust recovery.
Also, like EN but unlike LASSO, the CLOT formulation achieves the
grouping effect, wherein coefficients of highly correlated columns of
the measurement (or design) matrix are assigned roughly comparable values.
It is already known LASSO does not have the grouping effect.
Therefore the CLOT formulation combines the best features of both LASSO
(robust sparse recovery) and EN (grouping effect).

The CLOT formulation is a special case of another one called SGL
(Sparse Group LASSO) which was introduced into the literature
previously, but without any analysis
of either the grouping effect or robust sparse recovery.
It is shown here that SGL achieves robust sparse recovery,
and also achieves a version of the grouping effect in that
coefficients of highly correlated columns belonging to the same group of
the measurement (or design) matrix are assigned roughly comparable values.

\end{abstract}



\section{Introduction}\label{sec:intro}

The LASSO and the Elastic Net (EN) formulations are among the most
popular approaches for sparse regression and compressed sensing.
In this section, we briefly review these two problems and their
current status, so as to provide the background for the remainder of the paper.

\subsection{Sparse Regression}\label{ssec:sparse-reg}

In sparse regression, one is given a measurement matrix 
(also called a design matrix in statistics) $A \in \R^{m \times n}$
where $m \ll n$, together with a measurement or measured vector $y \in \R^m$.
The objective is to choose a vector $x \in \R^n$ such that $x$ is rather
sparse, and $Ax$ is either exactly or approximately equal to $y$.
The problem of finding the most sparse $x$ that satisfies $Ax = y$ 
is known to be NP-hard \cite{Natarajan95}; therefore it is necessary
to find alternate approaches.

For the sparse regression problem, the general approach is to determine
the estimate $\xh$ by solving the minimization problem
\be\label{eq:100}
\xh = \argmin_z \nmeusq{y - Az} \st \Rcal(z) \leq \g ,
\ee
or in Lagrangian form,
\be\label{eq:101}
\xh = \argmin_z \nmeusq { y - Az } + \l \Rcal (z) ,
\ee
where $\Rcal : \R^n \ap \R$ is known as a ``regularizer,'' and $\g,\l$
are adjustable parameters.
Different choices of the regularizer lead to different approaches.
With the choice $\Rcal_{{\rm ridge}} (z) = \nmeusq{z}$, the approach is known
as ridge regression \cite{Ridge}, which builds on earlier work
\cite{Tikhonov43}.
The LASSO approach \cite{Tibshirani-Lasso} results from choosing
$\Rcal_{{\rm LASSO}}(z) = \nmm{z}_1$, while the Elastic Net (EN) approach
\cite{Zou-Hastie05} results from choosing 
\be\label{eq:102a}
\Rcal_{{\rm EN}}(z) = \al_1 \nmm{z}_1 + \al_2 \nmeusq{z} ,
\ee
where $\al_1, \al_2$ are adjustable parameters.
For later use, we redefine the EN regularizer as
\be\label{eq:102}
\Rcal_{{\rm EN}}(z) = (1 - \mu) \nmm{z}_1 + \mu \nmeusq{z} ,
\ee
where
\bd
\mu = \frac{\al_2}{\al_1 + \al_2} \in [0,1]
\ed
is an adjustable parameter, and the constant $\al_1 + \al_2$ can be
absorbed into the Lagrange multiplier $\l$ in \eqref{eq:101}.
Note that the EN regularizer function interpolates ridge regression and
LASSO, in the sense that EN reduces to LASSO if $\mu = 0$ and to ridge
regression if $\mu = 1$.
A very general approach to regression using a convex regularizer is
given in \cite{NRWY12}.

The LASSO approach can be shown to return
a solution $\xh$ with no more than $m$ nonzero components, under
mild regularity conditions; 
see \cite{Osborne-Presnell-Turlach00}.
There is no such bound on the number of components of $\xh$ when EN is used.
However, when the columns of the matrix $A$ are highly correlated,
then LASSO chooses just one of these columns and ignores the rest.
Measurement matrices with highly correlated columns occur in many
practical situations, for example, in microarray measurements
of messenger RNA, otherwise known as gene expression data.
The EN approach was proposed at least in part to overcome this
undesirable behavior of the LASSO formulation.
It is shown in \cite[Theorem 1]{Zou-Hastie05} that if two columns
(say $i$ and $j$) of the matrix $A$ are highly correlated, then
the corresponding components $\xh_i$ and $\xh_j$ of the EN solution
are nearly equal.
This is known as the ``grouping effect,'' and the point is that EN
demonstrates the grouping effect whereas LASSO does not.

\subsection{Compressed Sensing}\label{ssec:comp-sens}

In compressed sensing, the objective is to \textit{choose} the
measurement matrix $A$ (which is part of the data in sparse regression),
such that whenever the vector $x$ is
nearly sparse, it is possible to nearly recover $x$ from noise-corrupted
measurements of the form $y = Ax + \eta$.
Let us make the problem formulation precise.
For this purpose we begin by introducing some notation.

Throughout, the symbol $[n]$ denotes the index set $\{ 1 , \ldots , n \}$.
The \textbf{support} of a vector $x \in \R^n$ is denoted by
$\supp(x)$ and is defined as
\bd
\supp(x) := \{ i \in [n] : x_i \neq 0 \} .
\ed
A vector $x \in \R^n$ is said to be \textbf{$k$-sparse} if
$| \supp(x) | \leq k$.
The set of all $k$-sparse vectors is denoted by $\SI_k$.
The \textbf{$k$-sparsity index} of a vector $x$ with respect to
a given norm $\nmm{\cdot}$ is defined as
\be\label{eq:103}
\s_k(x,\nmm{\cdot}) := \min_{z \in \SI_k} \nmm {x - z} .
\ee
It is obvious that $x \in \SI_K$ if and only if $\s_k(x,\nmm{\cdot}) = 0$
for every norm.

The general formulation of the compressed sensing problem given below
is essentially taken from \cite{Cohen-Dahmen-Devore09}.
Suppose that $A \in \R^{m \times n}$ is the ``measurement matrix,''
and $\D: \R^m \ap \R^n$ is the ``decoder map,''
where $m \ll n$.
Suppose $x \in \R^n$ is an unknown vector that is to be recovered.
The input to the decoder consists of 
$y = Ax + \eta$ where $\eta$ denotes the measurement noise, and
a prior upper bound in the form $\nmeu{\eta} \leq \e$ is available;
in other words, $\e$ is a known number.
In this set-up, the vector $\xh = \D(y)$ is the approximation to the
original vector $x$.
With these conventions, we can now state the following.

\begin{definition}\label{def:near}
Suppose $p \in [1,2]$.
The pair $(A,\D)$ is said to achieve \textbf{robust sparse recovery of
order $k$} with respect to $\nmp{\cdot}$ if there exist constants
$C$ and $D$ that might depend on $A$ and $\D$ but not on
$x$ or $\eta$, such that
\be\label{eq:104}
\nmp{ \xh - x } \leq \frac{1}{ k^{1 - 1/p} }
[ C \s_{k} ( x , \nm{\cdot}\nm_1 ) + D \e ] .
\ee
\end{definition}

The restriction that $p \in [1,2]$ is tied up with the fact that
the bound on the noise is for the Euclidean norm $\nmeu{\eta}$.
The usual choices for $p$ in \eqref{eq:104} are $p = 1$ and $p = 2$.

Among the most popular approaches to compressed sensing is $\ell_1$-norm
minimization, which was popularized in a series of papers, of which
we cite only \cite{Candes-Tao05,CRT06b,Candes08,Donoho06b}.
The survey paper \cite{DDEK12} has an extensive bibliography on the topic,
as does the recent book \cite{FR13}.
In this approach, the estimate $\xh$ is defined as
\be\label{eq:105}
\xh := \argmin_z \nmm{z}_1 \st \nmeu{Az-y} \leq \e .
\ee
Note that the above definition does indeed define a decoder map
$\D: \R^m \ap \R^n$.
In order for the above pair $(A,\D)$ to achieve robust sparse recovery,
the matrix $A$ is chosen so as to satisfy a condition defined next.

\begin{definition}\label{def:RIP}
A matrix $A \in \R^{m \times n}$ is said to satisfy
the \textbf{Restricted Isometry Property (RIP)}
of order $k$ with constant $\d_k$ if
\be\label{eq:105a}
(1 - \d_k) \nmeusq{u} \leq
\nmeusq { Au } \leq (1 + \d_k) \nmeusq{u},  \fa u \in \SI_k .
\ee
\end{definition}

Starting with \cite{Candes-Tao05}, several papers have derived sufficient
conditions that the RIP constant of the matrix $A$ must satisfy in order
for $\ell_1$-norm minimization to achieve robust sparse recovery.
Recently, the ``best possible'' bound has been proved in \cite{CZ14}.
These results are stated here for the convenience of the reader.

\begin{theorem}\label{thm:CZ1}
(See \cite[Theorem 2.1]{CZ14})
Suppose $A$ satisfies the RIP of order $tk$ for some number $t \geq 4/3$
such that $tk$ is an integer, with $\d_{tk} < \sqrt{(t-1)/t}$.
Then the recovery procedure in \eqref{eq:105} achieves robust sparse
recovery of order $k$.
\end{theorem}

\begin{theorem}\label{thm:CZ2}
(See \cite[Theorem 2.2]{CZ14})
Let $t \geq 4/3$.
For all $\g > 0$ and all $k \geq 5/\g$, there exists a matrix $A$
satisfying the RIP of order $tk$ with constant $\d_{tk} \leq
\sqrt{(t-1)/t} + \g$ such that the recovery procedure in \eqref{eq:105}
fails for some $k$-sparse vector.
\end{theorem}

Observe that the Lagrangian formulation of the LASSO approach is
\bd
\xh := \argmin_z [ \nmeusq{Az-y} + \l \nmm{z}_1 ] ,
\ed
whereas the Lagrangian formulation of \eqref{eq:105} is
\bd
\xh := \argmin_z [ \nmm{z}_1 + \beta \nmeu{Az-y} ] ,
\ed
which is essentially the same as the Lagrangian formulation of
\bd
\xh = \argmin_z \nmeu{Az-y} \st \nmm{z}_1 \leq \g .
\ed
This last formulation of sparse regression is known as ``square-root LASSO''
\cite{Belloni-et-al14}.
Therefore the community refers to the approach to compressed sensing given
in \eqref{eq:105} as the LASSO, though this may not be
strictly accurate.

\subsection{Compressed Sensing with Group Sparsity}\label{ssec:group}

Over the years some variants of LASSO have been proposed for compressed
sensing, such as the Group LASSO (GL) \cite{Yuan-Lin-Group-Lasso}
and the Sparse Group LASSO (SGL) \cite{SFHT13}.
In the GL formulation, the index set $\{ 1 , \ldots , n \}$ is partitioned
into $g$ disjoint sets $G_1 , \ldots , G_g$,
and the associated norm is defined as
\be\label{eq:107}
\nmm{z}_{{\rm GL}} := \sum_{i=1}^g \nmeu { z_{G_i} } ,
\ee
where $z_{G_i}$ denotes the projection of the vector $z$ onto the
components in $G_i$.
The notation is intended to remind us that the norm depends on
the specific partitioning $\G$.
Some authors divide the term $\nmeu { z_{G_i} }$ by $|G_i|$,
but we do not do that.
A further refinement of GL is the sparse group LASSO (SGL),
in which the group structure
is as before, but the norm is now defined as
\be\label{eq:108}
\nmgmu{z} := \sum_{i=1}^g ( 1 - \mu ) \nm z_{G_i} \nm_1
+ \mu \nmeu { z_{G_i} } ,
\ee
where as before $\mu \in [0,1]$.
If $x \in \R^n$ is an unknown vector, then recovery of $x$
is attempted via
\be\label{eq:109}
\xh = \argmin_z \nmm{z}_{{\rm GL}} \st \nmeu{Az-y} \leq \e 
\ee
in Group LASSO, and via
\be\label{eq:110}
\xh = \argmin_z \nmgmu{z} \st \nmeu{Az-y} \leq \e
\ee
in Sparse Group LASSO.

The main idea behind GL is that one is less concerned about the number
of nonzero components of $\xh$, and more concerned about the number
of distinct groups containing these nonzero components.
Therefore GL attempts to choose an estimate $\xh$ that has nonzero
entries in as few distinct sets as possible.
In principle, SGL tries to choose an estimate $\xh$ that not only
has nonzero components within as few groups as possible, but within
those groups, has as few nonzero components as possible.
Note that if $\mu = 0$, then SGL reduces to LASSO (because of the
summability of the $\ell_1$-norm), whereas if $\mu = 1$, then
SGL reduces to GL.
Note too that if $g = n$ and every set $G_i$ is a singleton $\{ i \}$,
then GL reduces to LASSO.

\subsection{Motivation and Contributions of the Paper}\label{ssec:mot}

Now we come to the motivation and contributions of the present paper.
The LASSO formulation is well-suited for compressed sensing (see
Theorem \ref{thm:CZ1}), but not so well-suited for sparse regression,
because it lacks the grouping effect.
The EN formulation is well-suited for sparse regression as it exhibits
the grouping effect, but it is not known whether it can achieve
compressed sensing.

The first result presented in the paper is that if the EN regularizer
of \eqref{eq:102} is used instead of the $\ell_1$-norm in  \eqref{eq:105}, then
the resulting approach \textit{does not achieve} robust sparse recovery
unless $m \geq n/4$, that is, the number of measurements grows linearly
with respect to the size of the vector.
This would not be considered ``compressed'' sensing.
This led us to formulate another regularizer, namely
\be\label{eq:111}
\nmcmu{z} = (1 - \mu) \nmm{z}_1 + \mu \nmeu{z} .
\ee
Note that, while the EN regularizer in \eqref{eq:102} is a convex function,
it is not a norm.
In contrast, $\nmcmu{\cdot}$ is not just convex but is also a norm.
Also, the EN regularizer in its original form in \eqref{eq:102a} is
intended to have \textit{two} adjustable parameters.
Our intent is that, in compressed sensing applications,
the constant $\mu$ in \eqref{eq:111} is a
\textit{fixed constant}, and not intended to be varied.
Therefore, if the $\ell_1$-norm in \eqref{eq:105} is replaced by
$\nmcmu{\cdot}$, then there is only one adjustable parameter, namely
the Lagrange multiplier associated with the constraint.
The same remark applies also to GL and SGL, that is, \eqref{eq:109}
and \eqref{eq:110} respectively.
We refer to $\nmcmu{\cdot}$ as the CLOT norm, with CLOT standing for Combined
L-One and Two.
It is shown that the CLOT norm combines the best features of both LASSO
and EN, in that
\bit
\item When the CLOT norm is used as the regularizer in sparse regression,
the resulting solution exhibits the grouping effect.
\item When the $\ell_1$-norm is replaced by the CLOT norm in \eqref{eq:105},
the resulting solution achieves robust sparse recovery if the matrix $A$
satisfies the RIP.
\item Moreover, if $\mu$ in CLOT is set to zero so that CLOT becomes
LASSO, the bound on the RIP constant reduces to the ``best possible''
bound in Theorem \ref{thm:CZ1}.
\eit

Clearly the CLOT norm is a special case of the SGL norm with
the entire index set $[n]$ being taken as a single group (though the
adjective ``sparse'' is no longer appropriate).
This led us to explore whether the SGL norm achieves either grouping
effect or robust sparse recovery.
We are able to show that SGL does indeed achieve both.

Now we place these contributions in perspective.
There is empirical evidence to support the belief that both the GL and the SGL
formulations work well for compressed sensing.
However, until the publication of a companion paper by 
a subset of the present authors \cite{MV-Eren-Bounds16},
there were no proofs that either of these formulations
achieved robust sparse recovery.
In \cite{MV-Eren-Bounds16}, it is shown that both the GL and SGL
formulations achieve robust sparse recovery \textit{provided the group
sizes are sufficiently small}.
This restriction on group sizes is removed in the present paper.
Moreover, so far as the authors are aware, until now there are no results on
the grouping effect for either of these formulations.
In the present paper, it is shown that
if two columns of the measurement
matrix $A$ \textit{that belong to the same group} are highly correlated,
then the corresponding components of the estimate $\xh$ have nearly
equal values.
However, if two columns that belong to different groups are highly
correlated, then their coefficients need not be nearly equal.
From the standpoint of applications, this is a highly desirable property.
To illustrate, suppose the groups represent biological pathways.
Then one would wish to assign roughly similar weights to genes in the
same pathway, but not necessarily to those in disjoint pathways.

Thus the contributions of the present paper are:
\bit
\item To show that the EN \textit{does not achieve} robust sparse recovery.
\item To show that both the CLOT and SGL formulations achieve both robust sparse
recovery as well as the grouping effect.
\item To derive a condition under which CLOT achieves robust sparse recovery,
which reduces to the ``best possible'' condition in Theorem \ref{thm:CZ1}
when $\mu$ is set to zero, so that CLOT becomes LASSO.
\eit
Taken together, these results might indicate that CLOT and SGL are
attractive alternatives to the LASSO and EN formulations.

\section{Main Theoretical Results}\label{sec:main}

This section contains the main contributions of the paper.
We begin by showing in Section \ref{ssec:rel}
that the solution paths of EN and CLOT are identical
if both $\l_1$ and $\l_2$ are treated as adjustable parameters.
Therefore further research would be needed to establish whether
CLOT offers any advantages over EN in numerical performance in
sparse regression.
Then we present several theoretical advantages of CLOT over both EN and LASSO.
First it is shown in Section \ref{ssec:lack}
that the EN approach does not achieve robust
sparse recovery, and is therefore not suitable for compressed
sensing applications.
Next, it is shown in Section \ref{ssec:corr} that the SGL formulation assigns
nearly equal weights to highly correlated features \textit{within the same group},
though not necessarily to highly correlated features from different groups.
It follows as a corollary that CLOT assigns nearly equal weights to highly
correlated features.
Then it is shown in Section \ref{ssec:robust}
that the SGL formulation achieves robust sparse recovery.
The contents of a companion paper by a subset of the present authors
\cite{MV-Eren-Bounds16} establish that SGL achieves robust sparse
recovery of order $k$
\textit{provided that each group size is smaller than $k$}.
There is no such restriction here.
It follows as a corollary that CLOT also achieves robust sparse recovery.

\subsection{Relationship Between Solution Paths of EN and CLOT}\label{ssec:rel}

\def\lt{\tilde{\lambda}}
\def\lh{\hat{\lambda}}

In this subsection, it is shown that if both $\mu$ and $\l$ are tuned via
cross-validation in \eqref{eq:102a}, then the solution paths of
CLOT are identical to those of EN when both $\l_1$ and $\l_2$
are tuned.
However, it is shown via an example that if $\mu$ is kept fixed and
\textit{only} $\l$ is tuned in \eqref{eq:102a}, then CLOT and EN have
different solution paths.

Towards this end, we rewrite the CLOT formulation with
both $\mu$ and $\l$ being tuned in the form
\be\label{eq:33}
\xh_{{\rm CLOT}} := \argmin_z [ \nmeusq{y - Ax} + \l_1 \nmm{z}_1
+ \l_2 \nmeu{z} ] .
\ee
It is easy to see that the transformation
\bd
\mu = \frac{ \l_2 }{ \l_1 + \l_2 } , \l = \l_1 + \l_2 
\ed
maps \eqref{eq:33} into \eqref{eq:102a}.
In the other direction, we would define
\bd
\l_1 = (1 - \mu) \l , \l_2 = \mu \l .
\ed
We are grateful to one of the reviewers for pointing out the result
as described in Theorem \ref{thm:31}, and providing a proof.

\begin{theorem}\label{thm:31}
Given $y \in \R^m$ and $A \in \R^{m \times n}$, define two vectors:
\bd
\xh_{{\rm CLOT}}(\l_1,\lt_2) = \argmin_z [ \nmeusq{y - Az} + \l_1 \nmm{z}_1
+ \lt_2 \nmeu{z} ] ,
\ed
\bd
\xh_{{\rm EN}}(\l_1,\lh_2) = \argmin_z [ \nmeusq{y - Az} + \l_1 \nmm{z}_1
+ \lh_2 \nmeusq{z} ] .
\ed
Then for each fixed $\l_1 > 0$ and each $\lt_2 > 0$, there exists a $\lh_2 > 0$
such that
\bd
\xh_{{\rm CLOT}}(\l_1,\lt_2) = \xh_{{\rm EN}}(\l_1,\lh_2) ,
\ed
and vice versa.
\end{theorem}

\textbf{Proof:}
We begin with the following rather obvious observation.
Suppose $f(\cdot)$ and $g(\cdot)$ are convex functions, and consider
two problems:
\bd
(P1) \xh_1(\l) = \argmin_z [ f(z) + \l g(z) ] ,
\ed
\bd
(P2) \xh_2(c) = \argmin_z f(z) \st g(z) \leq c .
\ed
Then for each $\l$ there exists a $c$ such that
$\xh_1(\l) = \xh_2(c)$, and vice versa.
To establish this, write down the optimality conditions for the
two problems, with $\partial f(\cdot), \partial g(\cdot)$ denoting the
subgradient sets of $f(\cdot),g(\cdot)$ respectively.
Then a necessary and sufficient condition for $\xh_1(\l)$ to be the
solution of (P1) is:
\be\label{eq:35}
\bz \in \partial f( \xh_1(\l)) + \l \partial g(\xh_1(\l) ) ,
\ee
where $\bz$ denotes the zero vector.
Similarly, for (P2) the necessary and sufficient conditions are the
existence of a constant $\l^*$ such that
\be\label{eq:36}
\bz \in \partial f( \xh_1(\l^*)) + \l \partial g(\xh_1(\l^*) ), 
\mbox{ and } \l^* ( g(\xh_1(\l^*)) - c) = 0 .
\ee
Suppose \eqref{eq:35} holds; then
\eqref{eq:36} holds with $c = g(\xh_1(\l))$.
Conversely, suppose \eqref{eq:36} holds;
then \eqref{eq:35} holds with $\l = \l^*$.

Now apply this reasoning with $f(z) = \nmeusq{y - Az} + \l_1 \nmm{z}_1$,
$g_1(z) = \nmeu{z}$, $g_2(z) = \nmeusq{z}$.
Then each $\xh_{{\rm CLOT}}(\l_1,\lt_2)$ equals the minimizer of
$f(z)$ subject to $\nmeu{z} \leq c$ for some $c$,
while each $\xh_{{\rm EN}}(\l_1,\lh_2)$ equals the minimizer of
$f(z)$ subject to $\nmeusq{z} \leq c'$ for some $c'$.
However, it is obvious that
\bd
[ \nmeu{z} \leq c ] \iff [ \nmeusq{z} \leq c^2 ] .
\ed
Therefore the theorem is proved.
\halmos

\subsection{Lack of Robust Sparse Recovery of the Elastic Net Formulation}
\label{ssec:lack}

The first result of this section shows that
EN formulation does not achieve robust sparse recovery, and therefore
is not suitable for compressed sensing applications.

\begin{theorem}\label{thm:EN}
Suppose a matrix $A \in \R^{m \times n}$ has the following property:
There exist constants $C$ and $D$ such that,
whenever $y = Ax + \eta$ for some $x \in R^n$ and $\eta \in \R^m$ with
$\nmeu{ \eta } \leq \e$, the solution
\bd
\xh_{{\rm EN}} := \argmin_{z \in \R^n} 
[ ( 1 - \mu ) \nmm{z}_1 + \mu \nmeusq{z} ] \st \nmeu { y - Az } \leq \e 
\ed
satisfies
\be\label{eq:1}
\nmeu { \xh_{{\rm EN}} - x } \leq C \s_k(x,\nmm{\cdot}_1 ) + D \e .
\ee
Then
\be\label{eq:2}
m \geq n/4 .
\ee
\end{theorem}

\textbf{Proof:}
Let $\N(A)$ denote the null space of the matrix $A$, that is, the set
of all $h \in \R^n$ such that $Ah = 0$.
Let $h \in \N(A)$ be arbitrary, and let $\L \seq \{ 1 , \ldots , n \}$
denote the index set of the $k$ largest components of $h$ by magnitude.
Therefore
\bd
\nmeu { \hlc } = \s_k(h, \nmeu{\cdot}) .
\ed
Next, \eqref{eq:1} implies that, if $\eta = 0$, then
$\xh_{{\rm EN}} = x$ for all $x \in \SI_k$.
In other words,
\bd
x = \argmin_{z \in \R^n} [ ( 1 - \mu ) \nmm{z}_1 + \mu \nmeusq{z} ] \st Az = Ax ,
\ed
or equivalently,
\be\label{eq:3}
( 1 - \mu ) \nmm{x}_1 + \mu \nmeusq{x} \leq ( 1 - \mu ) \nmm{z}_1 + \mu \nmeusq{z} ,
\fa z \in A^{-1}( \{ Ax \}) , \fa x \in \SI_k .
\ee
Now observe that, because $h \in \N(A)$, we have that
\bd
A \hl = - A \hlc ,
\ed
and more generally,
\bd
A( \beta \hl ) = A ( - \beta \hlc ) , \fa \beta > 0 .
\ed
Apply \eqref{eq:3} with $x = \beta \hl \in \SI_k$ and
$z = - A \beta \hlc \in A^{-1}( \{ Ax \}) $.
This leads to
\bd
( 1 - \mu ) \beta \nmm{ \hl }_1 + \beta^2 \mu \nmeusq{ \hl } \leq
( 1 - \mu ) \beta \nmm{ \hlc }_1 + \beta^2 \mu \nmeusq{ \hlc } .
\ed
Now divide both sides by $\beta^2 \mu$, and observe that,
for each fixed $\mu > 0$,
\bd
\frac{ 1 - \mu } { \beta \mu } \ap 0 \mbox{ as } \beta \ap \infty .
\ed
Therefore
\bd
\nmeusq { \hl } \leq \nmeusq { \hlc } ,
\mbox{ or } \nmeu { \hl } \leq \nmeu { \hlc } , \fa h \in \N(A) .
\ed
Next
\bd
\nmeu { h } \leq \nmeu { \hl } + \nmeu { \hlc } \leq 2 \nmeu { \hlc } ,
\fa h \in \N(A) .
\ed
Equivalently
\bd
\nmeu{ h } \leq 2 \s_k (h , \nmeu{\cdot} ) , \fa h \in \N(A) .
\ed
This is Equation (5.2) of Cohen-Dahmen-Devore (2009) with $C_0 = 2$.
As shown in Theorem 5.1 of that paper, this implies that $m \geq n/C_0^2
= n/4$, which is the desired conclusion.
\halmos

Note that the proof of Theorem \ref{thm:EN} remains valid even if we were
to allow the constant $\mu$ to be ``tuned,'' provided that it is bounded
away from zero.
In other words, the proof does not make use of the fact that $\mu$ is
a fixed constant.
Therefore even in the ``naive'' version of EN, in which the regularizer
is defined as in \eqref{eq:102a}, and both constants $\al_1$ and $\al_2$
are adjusted, robust sparse recovery requires that $m \geq n/4$
provided only that the ratio $\al_1/\al_2$ remains bounded away from zero
as both parameters are tuned.



\subsection{Grouping Property of the SGL and CLOT Formulations} \label{ssec:corr}

One advantage of the EN over LASSO is that the 
former assigns roughly equal weights to highly correlated features,
as shown in \cite[Theorem 1]{Zou-Hastie05}
and referred to as the grouping effect.
In contrast, if LASSO chooses one feature among a set of highly correlated
features, then generically it assigns a zero weight to all the rest.
To illustrate, if two columns of $A$ are identical, then in principle
LASSO could assign nonzero weights to both columns; however, the slightest
perturbation in the data would cause one or the other weight to become zero.
The drawback of this
is that the finally selected feature set is very sensitive 
to noise in the measurements. 
In this section
we prove an analog of \cite[Theorem 1]{Zou-Hastie05} for
SGL formulation.
Our result states that if two highly correlated features \textit{within
the same group} are chosen by SGL, then they will have roughly similar weights.
Since CLOT is a special case of SGL with the entire feature set treated
as one group, it follows that CLOT assigns roughly similar weights to
highly correlated features
in the entire set of features.
As a result, the final feature sets obtained using 
SGL or CLOT are less sensitive to noise in 
measurements than the ones obtained using LASSO.

\begin{theorem}\label{thm:51}
Let $y \in \R^m, A \in \R^{m \times n}$ be some vector and matrix
respectively.
Without loss of generality, suppose that $y$ is centered, i.e.\
$y^t \eb_m = 0$, where $\eb_m$ denotes a column vector consisting of $m$ ones,
and that $A$ is standardized, i.e.\ $\nmeu { a_j } = 1$
where $a_j$ denotes the $j$-th column of $A$.
Suppose $\mu > 0$, and let $\G$ denote a partition of $\{ 1 , \ldots , n \}$
into $g$ disjoint subsets.
Define
\be\label{eq:321}
\xh := \argmin_z \l \nmeusq { y - Az } + \nmgmu { z} ,
\ee
where $\l > 0$ is a Lagrange multiplier.
Suppose that, for two indices $i,j$ belonging to the same group $G_s$,
we have that $\xh_i \xh_j \neq 0$, where $\xh_i,\xh_j$ denote the
components of the vector $\xh$.
By changing the sign of one of the columns of $A$ if necessary,
it can be assumed that $\xh_i \xh_j > 0$.
Define
\bd
d(i,j) := \frac{ | \xh_i - \xh_j | } {2 \l \nmeu { y } } ,
\r(i,j) = a_i^t a_j .
\ed
Then
\be \label{eq:cor}
d(i,j) \leq \frac{ \sqrt{ 2 ( 1 - \r(i,j) )} \nmeu { \xh^s } } { \mu } ,
\ee
where $\xh^s$ is shorthand for $\xh_{G_s}$.
\end{theorem}

\textbf{Proof:}
Define
\begin{eqnarray*}
L(z , \mu ) & := & \l \nmeusq { y - Az } + \nmgmu {z} \\
& = & \l \nmeusq { y - Az }  
+ ( 1 - \mu ) \nm z \nm_1 + \mu \sum_{s=1}^g  \nmeu { z^s } , 
\end{eqnarray*}
where, as above, $z^s$ denotes $Z_{G_s}$.
Then $L$ is differentiable with respect to $z_i$ whenever $z_i \neq 0$.
In particular, since both $\xh_i$ and $\xh_j$ are nonzero by assumption,
it follows that
\bd
\left. \frac{ \partial L }{ \partial z_i } \right|_{z = \xh} =
\left. \frac{ \partial L }{ \partial z_j } \right|_{z = \xh} = 0.
\ed
Expanding the partial derivatives leads to
\bd
-2 \l a_i^t (y - A \xh) + ( 1 - \mu ) \sg (\xh_i)
+ \mu \frac{ \xh_i }{ \nmeu { \xh^s } } = 0 ,
\ed
\bd
-2 \l a_j^t (y - A \xh) + ( 1 - \mu ) \sg (\xh_j)
+ \mu \frac{ \xh_j }{ \nmeu { \xh^s } } = 0 .
\ed
Subtracting one equation from the other gives
\bd
2 \l ( a_j^t - a_i^t ) ( y - A \xh )
+ \mu \frac{ \xh_i - \xh_j }{ \nmeu { \xh^s } } = 0 .
\ed
Hence
\begin{eqnarray*}
\frac { | \xh_i - \xh_j | } { \nmeu { \xh^s } }
& = & \frac { 2 \l }{ \mu } | ( a_j^t - a_i^t ) ( y - A \xh ) |  \\
& \leq & \frac { 2 \l }{ \mu } \nmeu { a_j - a_i } \nmeu { y - A \xh }
 \\
& \leq & \frac { 2 \l }{ \mu } \sqrt{ 2 ( 1 - \r(i,j) )} \nmeu { y } . 
\end{eqnarray*}
In the last step, we use the fact that
\bd
\nmeu { y - A \xh } \leq \nmeu { y - A \bz } = \nmeu { y } .
\ed
Rearranging gives
\bd
\frac { 1 } { 2 \l } \frac{ | \xh_i - \xh_j | } { \nmeu { y } }
\leq \frac{ \sqrt{ 2 ( 1 - \r(i,j) ) } \nmeu { \xh^s } } { \mu } ,
\ed
which is the desired conclusion.
\halmos

Let us illustrate the above result using the CLOT formulation.
In the case of CLOT formulation we have $g=1$, $\G=\{G_1\}$, $G_1=\{1,\cdots,n\}$, 
and the inequality in \eqref{eq:cor} becomes
\be \label{eq:421b}
d(i,j) \leq \frac{ \sqrt{ 2 ( 1 - \r(i,j) )} \nmeu { \xh } } { \mu }, 
\ee
where $\xh$ is the solution of the CLOT formulation, and
\be \label{eq:421a}
d(i,j) = \frac{ | \xh_i - \xh_j | } {2 \l \nmeu { y } }. 
\ee
Now suppose that two indices $i$ and $j$ are highly correlated
such that $\r(i,j) \approx 1$, so that the right hand side of the
inequality in \eqref{eq:421b} is almost equal to zero.
Combining this with \eqref{eq:421a} we can conclude $\xh_i  \approx \xh_j$, so
CLOT assigns similar weights to highly correlated variables.
\halmos

Though the focus of the present paper is not on the GL
formulation, we digress briefly to discuss the implications of
 Theorem \ref{thm:51} for GL.
This theorem also implies that the GL formulation exhibits the grouping
effect, because GL is a special case of SGL with $\mu = 1$.
Indeed, it can be observed from \eqref{eq:cor} that the bound on
the right side is minimized by setting $\mu = 1$, that is, using GL
instead of SGL.
This is not surprising, because SGL not only tries to minimize the
number of distinct groups containing the support of $\xh$, but
within each group, tries to choose as few elements as possible.
Thus, within each group, SGL inherits the weaknesses of LASSO.
Thus one would expect that, within each group, the feature set
chosen would become more sensitive as we decrease $\mu$.

\subsection{Robust Sparse Recovery of the SGL and CLOT Formulations}
\label{ssec:robust}


In this subsection, we present some sufficient conditions for the
SGL and CLOT formulations to achieve robust sparse recovery.
When CLOT is specialized to LASSO by setting $\mu = 0$, the sufficient
condition reduces to the ``tight'' bound given in Theorem \ref{thm:CZ1}.

Recall the definitions.
The CLOT norm with parameter $\mu$ is given by
\bd
\nmcmu{z} := (1 - \mu) \nmm{z}_1 + \mu \nmeu{z} ,
\ed
while the SGL norm is given by
\bd
\nmm{z}_{SGL,\mu} := \sum_{i = 1}^{g}[(1 - \mu) \nmm{z_{G_i} }_1 + \mu \nmeu{z_{G_i} } ]
\ed
Recall also the problem set-up.
The measurement vector $y$ equals $Ax + \eta$ where $\nmeu{\eta} \leq \e$,
a known upper bound.
The recovered vector $\xh$ is defined as
\be\label{eq:3112}
\xh = \argmin_z \nmgmu{z} \st \nmeu{ Az - y} \leq \e
\ee
if SGL is used, and as
\be\label{eq:3113}
\xh = \argmin_z \nmcmu{z} \st \nmeu{ Az - y} \leq \e
\ee
if CLOT is used.

\begin{definition}\label{def:RNSP}
A matrix $A \in \R^{m \times n}$ is said to satisfy the \textbf{$\ell_2$
robust null space property (RNSP)} if there exist constants $\r \in (0,1)$
and $\t \in \R_+$ such that, for all sets $S \seq [n]$ with $|S| \leq k$, 
we have
\be\label{eq:D}
\nmeu{h_S} \leq \frac{\r}{ \sqrt{k} } \nmm{h_{S^c} }_1 +
\frac{\t} { \sqrt{k} } \nmeu{ Ah } .
\ee
\end{definition}

This property was apparently first introduced in \cite[Definition 4.21]{FR13}.
Note that the definition in \cite{FR13} has just $\t$ in place of 
$\t/\sqrt{k}$.
It is easy to show that, if \eqref{eq:D} holds, then
\bd
\nmm{h_S}_1 \leq \r \nmm{h_{S^c} }_1 + \t \nmeu{ Ah } .
\ed

The following result is established in \cite{MV-Ranjan16} in the context
of group sparsity, but is new even for conventional sparsity.
The reader is directed to that source for the proof.

\begin{theorem}\label{RIP_A}
(\cite[Theorem 5]{MV-Ranjan16})
Suppose that, for some number $t > 1$, the matrix $A$ satisfies the RIP of
order $tk$ with constant $\d_{tk} < \sqrt{ (t-1)/t }$.
Let $\d$ be an abbreviation for $\d_{tk}$, and define the constants
\be\label{eq:319}
\nu := \sqrt{t(t-1)} - (t-1) \in (0,0.5) ,
\ee
\be\label{eq:3110}
a := [ \nu ( 1 - \nu) - \d ( 0.5 - \nu + \nu^2) ]^{1/2} ,
\ee
\be\label{eq:3110a}
b := \nu ( 1 - \nu ) \sqrt{1+\d} ,
\ee
\be\label{eq:3110b}
c := \left[ \frac { \d \nu^2 }{ 2(t-1) } \right]^{1/2} .
\ee
Then $A$ satisfies the $\ell_2$-robust null space property with
\be\label{eq:3111}
\r = \frac{c}{a} < 1 , \t = \frac{b \sqrt{k} }{a^2} .
\ee
\end{theorem}

In Theorem \ref{thm:SGL-Bounds} below, it is assumed that the matrix $A$
satisfies the RIP of order $tk$ with $\d_{tk} < \sqrt{(t-1)/t}$,
in accordance with Theorem \ref{thm:CZ1}.
With this assumption,
we prove bounds on the residual error $\nmm{\xh - x}_1$ with SGL;
the bounds for CLOT can be obtained simply by setting $g = 1$ in the
SGL bounds.
Note that, once bounds for $\nmm{\xh - x}_1$
are proved,
it is possible to extend the bounds to $\nmp{\xh - x}$ for all $p \in [1,2]$;
see \cite[Theorem 4.22]{FR13}.

\begin{theorem}\label{thm:SGL-Bounds}
Suppose $x \in \R^n$ and that $A \in \R^{m \times n}$ satisfies
the RIP of order $tk$ with constant $\d = \d_{tk} < \sqrt{(t-1)/t}$,
and define constants $\r, \t$ as in \eqref{eq:3111}.
Suppose that
\be\label{eq:MV1}
\mu < \frac{1-\r}{\sqrt{g} (1+\r)} .
\ee
Define
\be\label{eq:MV3}
\g = \frac { \mu \sqrt{g} }{1 - \mu } .
\ee
With these assumptions,
\be\label{eq:Bs}
\nmm{\xh - x}_1 \leq C \s_k(x, \nmm{\cdot}_1 ) + D \e ,
\ee
\be\label{eq:As}
\nmp{\xh - x} \leq \frac{1}{ k^{1 - 1/p} } [ (1+\r) C \s_k(x, \nmm{\cdot}_1 )
+ ( (1+\r) D + 2\t) \e ] , \fa p \in (1,2] ,
\ee
where
\be\label{eq:Cs}
C = \frac{2 ( 1+\r )}{ (1-\g) - (1+\g) \r } ,
D = \frac{4 \t}{ (1-\g) - (1+\g) \r } .
\ee
\end{theorem}

The proof of the above theorem is presented in an appendix, due to its
length.

In the above theorem, we started with the restricted isometry constant
$\d$ and computed an upper bound on $\mu$ in order for SGL and CLOT
to achieve robust sparse recovery.
As $\d_{tk}$ gets closer to the limit $\d_{tk} < \sqrt{(t-1)/t}$
(which is known to be the best possible in view of Theorem \ref{thm:CZ2}),
the limit on $\mu$ would approach zero.
%
It is also possible to start with $\mu$ and find an upper bound on $\d$,
by rearranging the inequalities.
As this involves just routine algebra, we simply present the final bound.
Given the number $t \geq 4/3$, define $\nu$ as in \eqref{eq:319}, and define
\be\label{eq:3115a}
\th_1 := \nu ( 1 - \nu ) \in (0,0.25) ,
\th_2 := 0.5 - \th_1 \in (0.25,0.5) ,
\ee
\be\label{eq:3115b}
\th_3 := \frac{ \nu^2 }{ 2(t-1) } .
\ee
Given $\mu > 0$, define $\g$ as in \eqref{eq:MV3}, and define
\bd
\bar{\r} := \frac{1 - \g}{ 1 + \g } .
\ed
If the matrix
$A$ satisfies the RIP of order $tk$ with a constant $\d = \d_{tk}$,
then SGL achieves robust sparse recovery of order $k$ provided
\be\label{eq:3116}
\d  < \rbar^2 \frac{ \th_1 }{ \th_3 + \rbar^2 \th_2 } ,
\ee

\section{Numerical Examples}

In this section we present two simulation studies, to demonstrate the
the grouping effect of CLOT and the lack of robust sparse recovery of EN,
respectively.

\subsection{Grouping Effect of CLOT}

To illustrate that CLOT demonstrates the grouping effect as does EN 
(see Theorem \ref{thm:51}), we ran the same example as at the end of
\cite[Section 5]{Zou-Hastie05}.
Specifically, we chose $Z_1$ and $Z_2$ to be two independent $U(0,20)$
random variables, and the observation $y$ as $N(Z_1 + 0.1 Z_2 , 1)$.
The six observations were
\bd
x_1 = Z_1 + \e_1 , x_2 = -Z_1 + \e_2 , x_3 = Z_1 + \e_3 ,
x_4 = Z_2 + \e_4 , x_5 = -Z_2 + \e_5 , x_6 = Z_2 + \e_6 ,
\ed
where the $\e_i$ are i.i.d.\ $N(0,1/16)$.
The objective was to express $y$ as a linear combination of $x_1$
through $x_6$.
In other words, we wished to express $y = X \bbeta$, where $X$ is the
matrix with the $x_i$ as columns, and $\bbeta$ is a $6 \times 1$ vector.
Ideally the outcome should be to assign high weights to the correlated
group $x_1 , x_2 , x_3$ and low weights to $x_4 , x_5 , x_6$.
Therefore, if $\bbeta$ denotes the six-dimensional coefficient vector,
we should have that
\bd
\beta_1 \approx - \beta_2 \approx \beta_3 ,
\ed
and that $\beta_4$ through $\beta_6$ are much smaller than
$\beta_1$ through $\beta_3$.

The three algorithms LASSO, EN, and CLOT were implemented via the
Lagrangian formulation in \eqref{eq:101}, with $\l$ being increased.
Clearly, when $\l$ is sufficiently large, the optimal value of $\bbeta$ is
the zero vector.
The ``sufficiently large'' value of $\l$ varies from one algorithm to the next.
Figures \ref{fig:2} through \ref{fig:4}
show the solution trajectories of the three algorithms
as functions of $\l$, with $\mu$ set equal to $0.5$.
From Figures \ref{fig:2} and \ref{fig:3}, it is 
clear that both CLOT and EN very quickly reach the correct 
proportionalities between the large coefficient values,
which eventually become smaller and go to zero as $\l$ becomes larger.
In contrast, the LASSO solutions are quite inaccurate.

\bfig
\bc
\includegraphics[width=150mm]{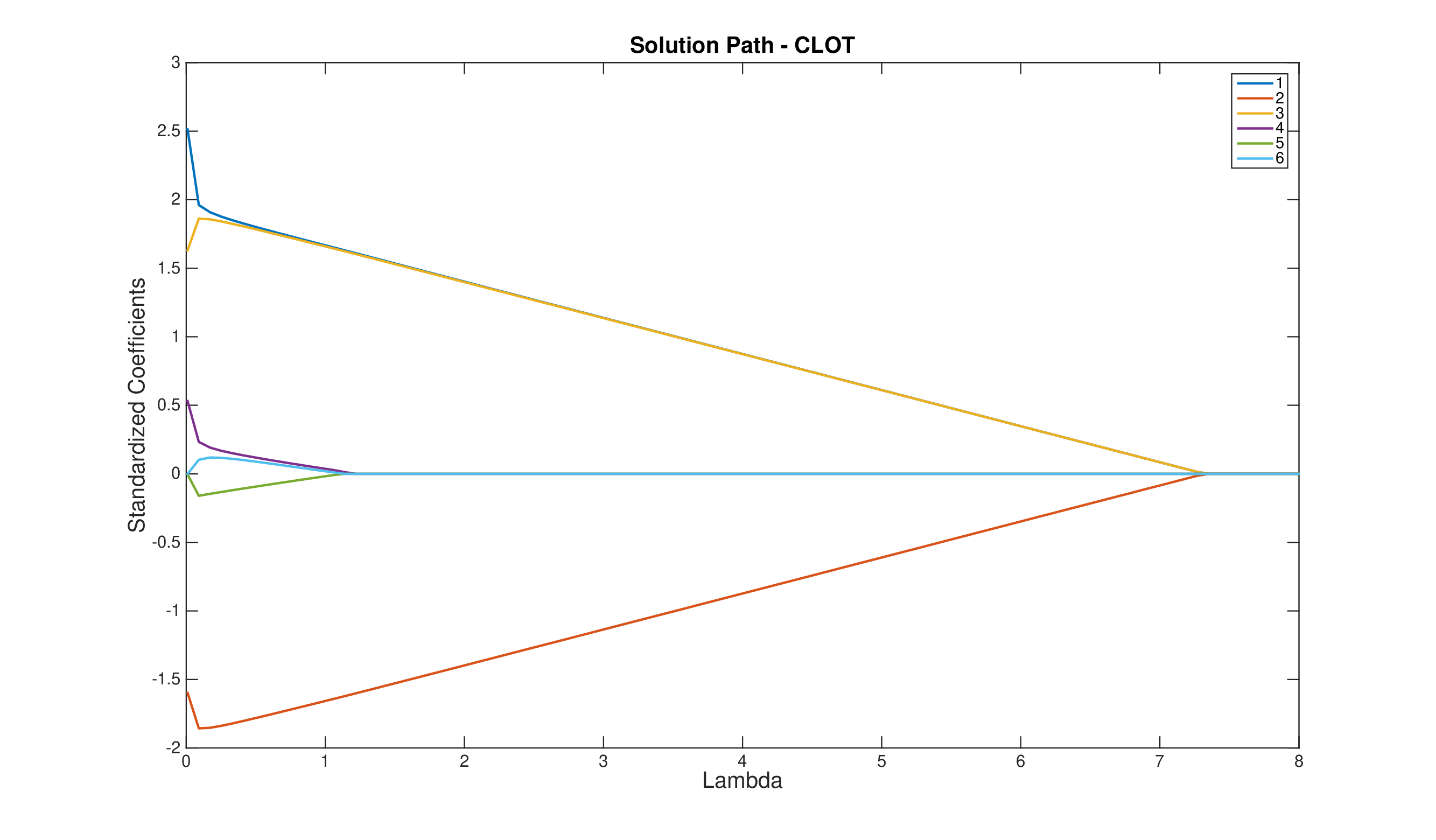}
\ec
\caption{Solution Paths for CLOT}
\label{fig:2}
\efig

\bfig
\bc
\includegraphics[width=150mm]{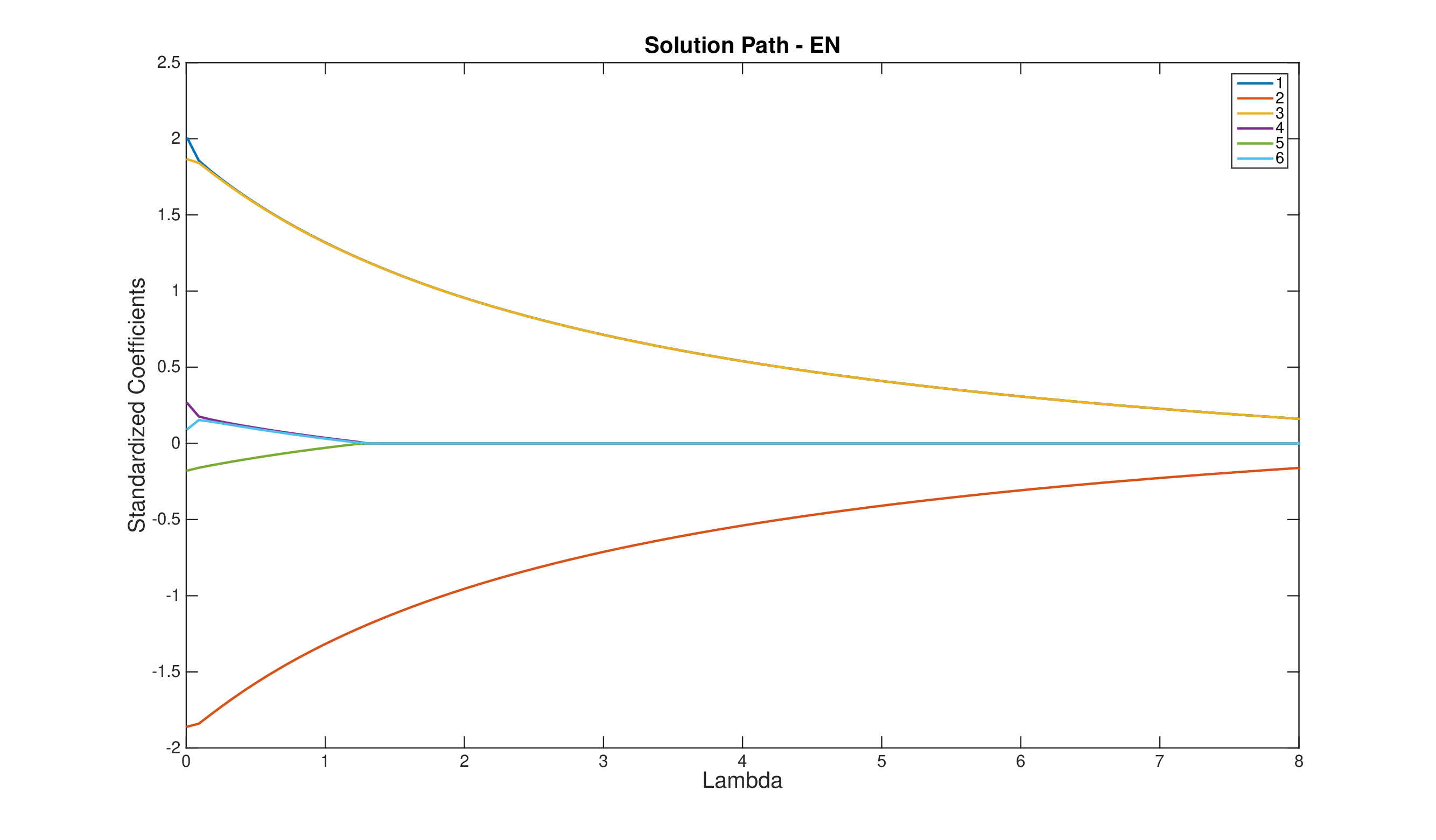}
\ec
\caption{Solution Path for EN}
\label{fig:3}
\efig

\bfig
\bc
\includegraphics[width=150mm]{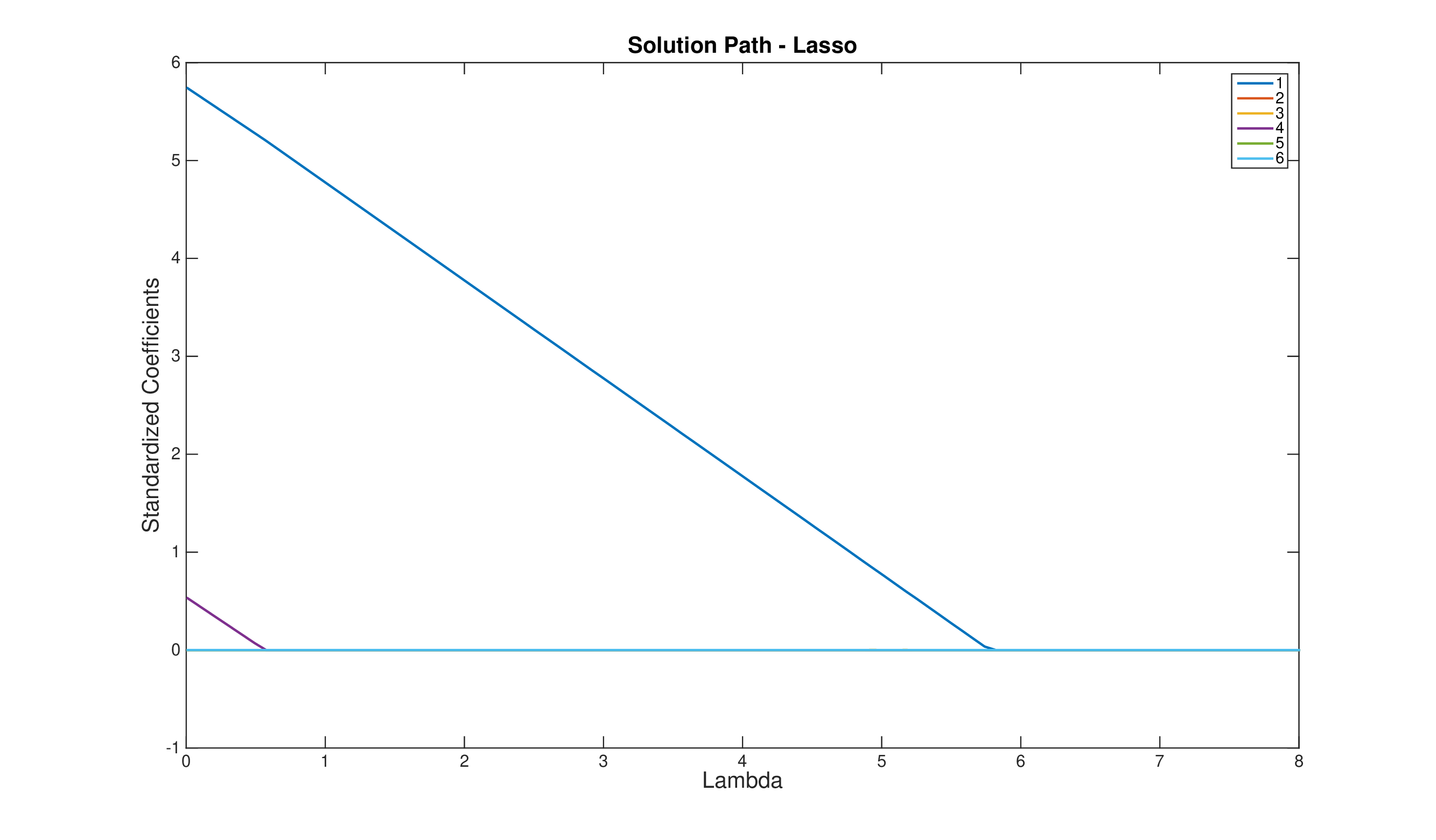}
\ec
\caption{Solution Path for LASSO}
\label{fig:4}
\efig

\subsection{Lack of Robust Sparse Recovery by EN} \label{sec:ex1}

In this subsection we illustrate Theorem \ref{thm:EN}.
In this set-up, $n = 4,000$ and $k = 3$.
The first three components of the vector $x$ are assigned values
at random using the {\tt Matlab} {\tt rand} function, which resulted in
$[ \ba{ccc} 0.8147 & 0.9058 & 0.1270 \ea ]$.
The remaining components of $x$ were set equal to zero.
We could have chosen not just the values but also the location of the
nonzero components at random; but this would have been just a permuation
of the above example.

Next, in order to achieve robust sparse recovery of order $k$,
following Theorem \ref{thm:CZ1} we needed to choose a $t \geq 4/3$
and then choose a matrix $A$ such that $A$ satisfied
the RIP of order $\lceil tk \rceil$ with constant $\d < \sqrt{(t-1)/t}$.
We chose $t = 1.5$, which resulted in $\lceil tk \rceil = 5$,
and $\d_5 < \sqrt{1/3} \approx 0.5774$.
We chose $\d_5 = 0.4$.
Therefore we had to choose an integer $m$ and a matrix $A \in \R^{m \times n}$
such that $A$ satisfied the RIP of order $5$ with constant $\d_5 \leq 0.4$.
Such a matrix was constructed using the deterministic procedure
suggested in \cite{Devore07}.
This required the choice of an integer $r \geq 2$
and a prime number $p$ such that 
\bd
p \geq \max \left\{ \frac{(\lceil tk \rceil -1 ) r} { \d } , n^{1/(r+1)}
\right\} = \max\{20, 15.87\} = 20 .
\ed
and led to a binary matrix of dimensions $p^2 \times p^{r+1}$.
The smallest prime number greater than $20$ is $p = 23$, and $m = p^2 = 529$.
Note that $m < n/4$.
The main advantages of the construction in \cite{Devore07} are that
(i) the construction is deterministic, (ii) the matrix $A$ is binary,
and (iii) only a fraction $1/p$ of the elements of $A$ are equal to
1, and the rest are equal to zero.
This makes computation very fast.

Once the matrix $A$ was chosen, we defined the measured vector as $y = Ax$;
that is, we did not introduce any measurement noise.
Then we computed estimates of $x$ using both CLOT and EN.
With $\d = 0.4$, the constant $\r$ defined in \eqref{eq:3111} becomes
$0.6551$, and the bound in \eqref{eq:MV1}, after substituting $g = 1$,
becomes $\mu < 0.2084$.
Therefore we chose $\mu = 0.2$, and defined
\bd
\xh_{{\rm CLOT}} = \argmin_z \nmcmu{z} \st Az = y ,
\ed
\bd
\xh_{{\rm EN}} = \argmin_z (1 - \mu) \nmm{z}_1 + \mu \nmeusq{z} \st Az = y ,
\ed
with the same value of $\mu = 0.2$ in both cases.
Then we replaced $x$ by $10^c x$ for $c = 1 , 2, 3, 4$ and recomputed
the estimates.
Because $m < n/4$, according to Theorem \ref{thm:51}, EN should fail as the
norm of the vector $x$ is increased.

The above constrained optimization problems were solved using the cvx
package in {\tt Matlab}, on an HP Pavilion laptop.
The actual results are shown in Table \ref{table:3}.
For compactness only the first three components of $\xh$ are shown.
As can be seen, when $c = 0$ or $c = 1$, both CLOT and EN
give the correct answer.
However, for larger values of $c$, the CLOT estimate simply got
multiplied by the same scale factor, whereas the EN estimate started
diverging from the true value with $c = 2$; the algorithm failed to converge
with $c = 4$.


\begin{table}[h]
\bc
\btab{|c|c|c|}
\hline
$c$ & $\xh_{{\rm CLOT}}$ & $\xh_{{\rm EN}}$ \\
\hline
0 &    0.8147 &  0.8147 \\
 &   0.9058 &  0.9058 \\
 &   0.1270 &  0.1270 \\
\hline
1 &    8.1472 &  8.1472 \\
 &   9.0579 &  9.0579 \\
  &  1.2699 &  1.2699 \\
\hline
 2 &  81.4724 & 24.1502 \\
 &  90.5792 & 31.8304 \\
 &  12.6987 &  4.7944 \\
\hline
3 &  814.7236 & 111.2433 \\
 & 905.7918 & 132.8940 \\
 & 126.9868 & 15.3818 \\
\hline
4 & $8.1472 \times 10^3$  &    NaN \\
 &    $9.0579  \times 10^3$ &    NaN \\
 &    $1.2699  \times 10^3$ &    NaN \\
\hline
\etab
\ec
\caption{Comparison of solutions of the CLOT and EN
algorithms as the unknown vector is scaled by successive powers of 10.
The CLOT output simply scales by the same factor, whereas the EN output
begins to be incorrect already when the scale factor is 100.
For a scale factor of $10^4$, the algorithm fails.}
\label{table:3}
\end{table}

\section{Discussion and Concluding Remarks}

In this paper we have introduced a new optimization formulation called
CLOT (Combined L-One and Two), wherein the regularizer is a convex
combination of the $\ell_1$- and $\ell_2$-norms.
This formulation differs from the Elastic Net (EN) formulation, in which the
regularizer is a convex combination of the $\ell_1$- and $\ell_2$-norm
\textit{squared}.
This seemingly simple modification has fairly significant consequences.
In particular, it is shown in this paper that the EN formulation
\textit{does not achieve} robust recovery of sparse vectors in the
context of compressed sensing, whereas the new CLOT formulation does so.
Also, like EN but unlike LASSO, the CLOT formulation achieves the
grouping effect, wherein coefficients of highly correlated columns of
the measurement (or design) matrix are assigned roughly comparable values.
It is noteworthy that LASSO does not have the grouping effect and EN
(as shown here) does not achieve robust sparse recovery.
Therefore the CLOT formulation combines the best features of both LASSO
(robust sparse recovery) and EN (grouping effect).

The CLOT formulation is a special case of another one called SGL
(Sparse Group LASSO) which was introduced into the literature
previously, but without any analysis
of either the grouping effect or robust sparse recovery
\cite{SFHT13}.
It is shown here that SGL achieves robust sparse recovery,
and also achieves a version of the grouping effect in that
coefficients of highly correlated columns of
the measurement (or design) matrix are assigned roughly comparable values,
\textit{if the columns belong to the same group}.

There are several papers in the literature that discuss LASSO-like
formulations for group sparsity; some of these are discussed here.
First, there is a companion paper by a subset of the present authors 
\cite{MV-Eren-Bounds16}, which studies the problem of robust sparse
recovery with SGL-like formulations, \textit{but with restrictions on
the group size}.
In contrast, in the present paper, there is no such restriction, which
is why the results derived here for the SGL formulation can be directly
applied to the CLOT formulation.
Second, for the case where several columns of the matrix $A$ are highly
correlated, \cite{BRGZ13} suggests a two-stage process whereby first
correlated columns are clustered, and second, a variant of LASSO is applied.
In a discussion of this paper, namely \cite{BW13}, all the LASSO variants
together with EN are run on various test data.
For the purposes of the present discussion, the salient observation is that
the EN formulation performed roughly as well -- no better and no worse --
compared almost all the LASSO variants.
Finally, in \cite{Geer14}, a general theory is presented whereby the
fully decomposable $\ell_1$-norm is replaced by a weakly decomposable norm,
and oracle bounds are derived for $\nmm { \xh - x}$ provided that the
index set $[n]$ is divided into an ``allowed set'' and its complement.
There is also some discussion of overlapping group decompositions.
Specifically, when an element of the index set $[n]$ appears in two groups,
the corresponding column of $A$ is simply replicated to remove the overlap.
However, if two columns of $A$ are identical (and normalized), then 
the RIP constant $\d_2$ would equal zero, as would $\d_k$ for $k \geq 2$.
Therefore, if the case of overlapping groups is handled in this manner,
then any analysis based on RIP would be infructuous.
The above discussion is quite cursory, and the reader may consult
these references for fuller details.

It would be worthwhile to study the behavior of SGL with overlapping groups.
There are variants of SGL with overlapping groups, provided they
satisfy some additional constraints; see \cite{Jenetton-et-al11,OJV-Over-GL11}
for example.
However, in a companion paper \cite{MV-Eren-Bounds16}, it is shown
that the assumptions of \cite{Jenetton-et-al11,OJV-Over-GL11}
still enforce a nonoverlap constraint, but in a nonobvious fashion.
As pointed out in the previous paragraph, the approach of introducing
duplicate columns into $A$ to eliminate overlap would render
any analysis based on RIP impossible.
Thus a suitable approach remains to be discovered.

\section*{Acknowledgements}

The authors thank two anonymous reviewers and the Handling Editor
for their careful reading
of earlier versions of this paper and many helpful comments.
Theorem \ref{thm:31} and its proof were supplied by one of the reviewers.
This research was supported by the National Science Foundation under Award
\#ECCS-1306630 and by the Cancer Prevention and Research Institute of Texas
(CPRIT) under Award No.\ RP140517, and by the Department of Science
and Technology, Government of India.

\section*{Appendix: Proof of Theorem \ref{thm:SGL-Bounds}}

\textbf{Proof:}
Hereafter we write $z^i$ instead of $z_{G_i}$ in the interests of brevity.

Define $h = \hat{x} - x \in \mathbb{R}^n$.
From the definition of the estimate, we have that
\bd
\nmm{x}_{SGL,\mu} \geq \nmm{ \hat{x} }_{SGL,\mu} = \nmm{x + h}_{SGL,\mu}
\ed
From the definition of the SGL norm, this expands to
\bd
(1 - \mu) \nmm{x}_1 + \mu \sum_{j = 1}^{g}\nmeu{x^j} \geq 
(1 - \mu) \nmm{x + h}_1 \\
+ \mu \sum_{j = 1}^{g} \nmeu{x^j + h^j} .
\ed
This can be rearranged as
\be\label{eq:1s}
(1 - \mu) ( \nmm{x}_1 - \nmm{x + h}_1 ) 
+ \mu \sum_{j=1}^g ( \nmeu{ x^j } - \nmeu{ ( x^j + h^j } ) \geq 0 .
\ee
We will work separately on each of the two terms separately.
First, by the triangle inequality, we have that
\bd
\nmm{x^j + h^j}_2 \geq \nmm{x^j}_2 - \nmm{h^j}_2 ,
\mbox{ or }
\nmm{h^j}_2 \geq \nmm{x^j}_2 - \nmm{x^j + h^j}_2 .
\ed
As a consequence,
\bd
\sum_{j = 1}^{g} \nmm{h^j}_2 \geq \sum_{j = 1}^{g} [\nmm{x^j}_2 - \nmm{x^j + h^j}_2]
\ed
From Schwarz' inequality, we get
\bd
\sum_{j = 1}^{g} \nmm{h^j}_2 \leq \sqrt{g} \nmeu{h} \leq \sqrt{g} \nmm{h}_1
\leq \sqrt{g} ( \nmm{h_S}_1 + \nmm{ h_{S^c} }_1 ) ,
\ed
for any $S \seq [n]$.
Combining everything gives
\be\label{eq:1t}
\sum_{j = 1}^{g} [\nmm{x^j}_2 - \nmm{x^j + h^j}_2]
\leq \sqrt{g} ( \nmm{h_S}_1 + \nmm{ h_{S^c} }_1 ) , 
\ee
for any subset $S \seq [n]$.
Second, for any subset $S \seq [n]$, the decomposability of $\nmm{\cdot}_1$
implies that
\bd
\nmm{x}_1 = \nmm{x_S} + \nmm{ x_{S^c} }_1 ,
\ed
while the triangle inequality implies that
\begin{eqnarray*}
\nmm{ x + h }_1 & = & \nmm{ x_S + h_S }_1 + \nmm{ x_{S^c} + h_{S^c} }_1 \\
& \geq & \nmm{x_S}_1 - \nmm{h_S}_1 + \nmm{ h_{S^c} }_1 - \nmm{ x_{S^c} }_1 .
\end{eqnarray*}
Therefore
\be\label{eq:1u}
\nmm{x}_1 - \nmm{ x + h }_1 \leq \nmm{h_S}_1 - \nmm{ h_{S^c} }_1 + 
2 \nmm{ x_{S^c} }_1 .
\ee
If we now choose $S$ to be the set corresponding to the $k$ largest
elements of $x$ by magnitude, then
\bd
\nmm{ x_{S^c} }_1 = \s_k(x,\nmm{\cdot}_1) =: \s_k .
\ed
With this choice of $S$, \eqref{eq:1u} becomes
\be\label{eq:1w}
\nmm{x}_1 - \nmm{ x + h }_1 \leq \nmm{h_S}_1 - \nmm{ h_{S^c} }_1 + 2 \s_k .
\ee
Substituting the bounds \eqref{eq:1t} and \eqref{eq:1u} into \eqref{eq:1s}
gives
\bd
0 \leq \mu \sqrt{g} ( \nmm{h_S}_1 + \nmm{ h_{S^c} }_1 ) 
+ (1 - \mu) ( \nmm{h_S}_1 - \nmm{ h_{S^c} }_1 + 2 \s_k ) . 
\ed
Now recall the definition of the constant $\g$ from \eqref{eq:MV3}.
Using this definition, the above inequality can be rearranged as
\bd
0 \leq \g ( \nmm{h_S}_1 + \nmm{ h_{S^c} }_1 )
+ ( \nmm{h_S}_1 - \nmm{ h_{S^c} }_1 + 2 \s_k ) .
\ed
and equivalently as
\be\label{eq:2s}
(1 - \g) \nmm{ h_{S^c} }_1 - (1+\g) \nmm{ h_S }_1 \leq 2 \s_k .
\ee
This is the first of two equations that we need.

Now we derive the second equation.
From Theorem \ref{RIP_A}, we know that the matrix $A$
satisfies the $l_2$-robust null space property, namely \eqref{eq:D}.
An application of Schwarz' inequality shows that
\bd
\nmm{h_S}_1 \leq \r \nmm{h_{S^c} }_1 + \t  \nmeu{ Ah } .
\ed
However, because both $x$ and $\xh$ are feasible for the optimization
problem in \eqref{eq:3112}, we get
\bd
\nmeu{ Ah } = \nmeu{ A \xh + \eta - (A x + \eta) } \leq 2 \e .
\ed
Substituting this bound for $\nmeu{Ah}$ gives us the second equation
we need, namely
\bd
\nmm{h_S}_1 \leq \r \nmm{h_{S^c} }_1 + 2 \t \e ,
\ed
or equivalently
\be\label{eq:1x}
- \r \nmm{h_{S^c} }_1 + \nmm{h_S}_1 \leq 2 \t \e .
\ee

The two inequalities \eqref{eq:2s} and \eqref{eq:1x} can be written compactly as
\be\label{mat_vecs}
M \begin{bmatrix}
\nmm{h_{S^c} }_1 \\
\nmm{h_S}_1
\end{bmatrix} \leq  
\begin{Bmatrix}
\begin{bmatrix}
2 \\
0
\end{bmatrix} \s_k +
\begin{bmatrix}
0 \\
2 \t
\end{bmatrix} \epsilon
\end{Bmatrix} ,
\ee
where the coefficient matrix $M$ is given by
\bd
M = \begin{bmatrix}
(1 - \g)  &&  -(1 + \gamma) \\
-\r  && 1
\end{bmatrix}
\ed
The matrix $M$ has positive diagonal elements (recall that 
$\g < 1$), and negative off-diagonal elements.
Therefore, if $\det(M) > 0$, then every element of $M^{-1}$ is positive,
in which we can multiply both sides of \eqref{mat_vecs} by $M^{-1}$.
Now 
\bd
\det(M) = 1 - \g -\r (1 + \g ) > 0 \iff \r < \frac{1-\g}{1+\g} .
\ed
Recall the definition of $\g$ from \eqref{eq:MV3}.
Now routine algebra shows that
\bd
\r < \frac{1-\g}{1+\g} \iff \r < \frac{1 - \mu \sqrt{g} }{1 + \mu \sqrt{g} }
\iff \mu < \frac{1-\r}{\sqrt{g}(1+\r)} ,
\ed
which is precisely \eqref{eq:MV1}.
Thus we can multiply both sides of \eqref{mat_vecs} by $M^{-1}$, which gives
\bd
\begin{bmatrix}
\nmm{h_{S^c} }_1 \\
\nmm{h_S}_1
\end{bmatrix} \leq 
\frac{1}{\det(M)}  \begin{bmatrix}
1  && (1 + \g) \\
\\
\r && (1 - \gamma)
\end{bmatrix}
\begin{Bmatrix}
\begin{bmatrix}
2 \\
0
\end{bmatrix} \s_k +
\begin{bmatrix}
0 \\
2 \t
\end{bmatrix} \e
\end{Bmatrix} .
\ed
Clearing out the matrix multiplication gives
\be\label{eq:1y}
\begin{bmatrix}
\nmm{h_{S^c} }_1 \\
\nmm{h_S}_1
\end{bmatrix} \leq
\frac{1}{\det(M)}
\left\{ \begin{bmatrix}
2 \\
2 \r
\end{bmatrix} \s_k +
\begin{bmatrix}
2 \t (1 + \g) \\
2 \t (1 - \g) 
\end{bmatrix} \e
\right\} .
\ee
Now the triangle inequality states that
\bd
\nmm{h}_1 \leq \nmm{h_{S^c} }_1 + \nmm{h_S}_1 = \begin{bmatrix}
1 && 1
\end{bmatrix}
\begin{bmatrix}
\nmm{h_{S^c} }_1 \\
\nmm{h_S}_1
\end{bmatrix} 
\ed
Substituting from \eqref{eq:1y} gives
\bd
\nmm{h}_1 \leq \frac{1}{\det(M)} [ 2 (1+\r) \s_k + 4 \t \e ] .
\ed
By substituting that $\det(M) = (1 - \g) - (1+\g) \r$, we get
\eqref{eq:Bs} with the constants as defined in \eqref{eq:Cs}.

To prove \eqref{eq:As}, suppose $p \in (1,2]$.
This part of the proof closely follows that of \cite[Theorem 4.22]{FR13},
except that we provide explicit values for the constants.
Let $\L_0$ denote the index set of the $k$ largest components of $h$
by magnitude.
Then
\bd
\nmp{h} \leq \nmp{ \hlo } + \nmp{ \hloc } .
\ed
We will bound each term separately.
First, by \cite[Theorem 2.5]{FR13} and \eqref{eq:Bs}, we get
\be\label{eq:Es}
\nmp{ \hloc } \leq \frac{1}{ k^{1 - 1/p} } \nmm{h}_1
\leq \frac{1}{ k^{1 - 1/p} } ( C \s_k + D \e ) .
\ee
Now we apply in succession H\"{o}lder's inequality, the robust null space
property, the fact that $\nmeu{Ah} \leq 2\e$, and \eqref{eq:Bs}.
This gives
\beq
\nmp{ \hlo } & \leq & k^{1/p - 1/2} \nmeu{ \hlo } \nonumber \\
& \leq & \frac{ k^{1/p - 1/2} }{ \sqrt{k} } [ \r \nmm{ \hloc }_1 
+ \t \nmeu{ Ah } ] \nonumber \\
& \leq & \frac{1}{ k^{1 - 1/p} } [ \r ( C \s_k + D \e ) + 2 \t \e ]
\nonumber \\
& = & \frac{1}{ k^{1 - 1/p} } [ \r C \s_k + ( \r D + 2 \t ) \e ] .
\label{eq:Fs}
\eeq
Adding \eqref{eq:Es} and \eqref{eq:Fs} gives \eqref{eq:As}.

\bibliography{Comp-Sens}

\bibliographystyle{IEEEtran}

\end{document}